\documentclass[a4paper,12pt]{article}

\usepackage{a4wide}

\usepackage{float} 
\usepackage[utf8]{inputenc}
\usepackage[unicode=true,hypertexnames=false]{hyperref}
\usepackage{amsmath,amssymb,amsthm,amsfonts}
\usepackage{dsfont} 
\usepackage{mathtools} 
\usepackage[linesnumbered,ruled]{algorithm2e}
\let\oldnl\nl
\newcommand{\nonl}{\renewcommand{\nl}{\let\nl\oldnl}}
\usepackage{xspace}
\usepackage[dvipsnames,table]{xcolor}
\usepackage{soul} 
\definecolor{lightblue}{rgb}{0.7,0.7,1}
\sethlcolor{lightblue} 

\usepackage{tikz}
\usetikzlibrary{decorations.pathreplacing,calc,tikzmark}
\usepackage{multirow}
\usepackage{booktabs}

\usepackage{pgfplots}
\pgfplotsset{compat=newest}
\pgfplotscreateplotcyclelist{runtimes}{%
    thick,blue,mark=triangle*\\
    thick,red,mark=*\\
}

\allowdisplaybreaks

\newcommand{\oea}{\mbox{$(1 + 1)$~EA}\xspace}
\newcommand{\oplea}{\mbox{$(1+\lambda)$~EA}\xspace}
\newcommand{\mpoea}{\mbox{$(\mu+1)$~EA}\xspace}

\newcommand{\mclea}{\mbox{$(\mu,\lambda)$~EA}\xspace}
\newcommand{\oclea}{\mbox{$(1,\lambda)$~EA}\xspace}

\newcommand{\onemax}{\textsc{OneMax}\xspace}

\newcommand{\leadingones}{\textsc{Leading\-Ones}\xspace}

\newcommand{\N}{{\mathbb N}}
\newcommand{\R}{{\mathbb R}}
\newcommand{\I}{\mathds{I}}

\newtheorem{theorem}{Theorem}
\newtheorem{lemma}[theorem]{Lemma}

\begin{document}

\title{Already Moderate Population Sizes Provably Yield Strong Robustness to Noise}

\author{Denis Antipov \\
		Optimisation and Logistics,\\ 
        School of Computer and Mathematical Sciences,\\ 
        The University of Adelaide \\
        Adelaide, Australia \\
		\and
		Benjamin Doerr \\
		Laboratoire d'Informatique (LIX), \\
		CNRS, \'Ecole Polytechnique, \\
		Institut Polytechnique de Paris \\
		Palaiseau, France \\
		\and
        Alexandra Ivanova \\
        HSE University, Skoltech \\
        Moscow, Russia \\
}

\maketitle
{\sloppy

\begin{abstract}
    Experience shows that typical evolutionary algorithms can cope well with stochastic disturbances such as noisy function evaluations.
    In this first mathematical runtime analysis of the $(1+\lambda)$ and $(1,\lambda)$ evolutionary algorithms in the presence of prior bit-wise noise, we show that both algorithms can tolerate constant noise probabilities without increasing the asymptotic runtime on the OneMax benchmark. For this, a population size $\lambda$ suffices that is at least logarithmic in the problem size $n$. The only previous result in this direction regarded the less realistic one-bit noise model, required a population size super-linear in the problem size, and proved a runtime guarantee roughly cubic in the noiseless runtime for the OneMax benchmark. Our significantly stronger results are based on the novel proof argument that the noiseless offspring can be seen as a biased uniform crossover between the parent and the noisy offspring. We are optimistic that this technique will find applications also in future mathematical runtime analyses of evolutionary algorithms.
\end{abstract}

\section{Introduction}
\label{sec:intro}

The mathematical runtime analysis has accompanied and supported the design and analysis of evolutionary algorithms (EAs) for more than 30 years. It has led to a deeper understanding of many important aspects of evolutionary computation~\cite{NeumannW10,AugerD11,Jansen13,ZhouYQ19,DoerrN20}.

While this area has also studied how EAs cope with noise, that is, a stochastically disturbed access to the true problem instance, the rigorous understanding of this aspect of evolutionary computation is rather limited. Most previous works regard simple algorithms like the \oea and show that these, without particular adjustments, can stand a moderate level of noise, but usually not more than one noisy fitness evaluation every roughly $n / \log(n)$ iterations for problems with bit-string representation of length~$n$. 

Some works have shown that larger population sizes improve the robustness to noise, but these still do not show a very satisfying picture. For example, in the work closest to ours, Gießen and Kötzing~\cite{GiessenK16} showed that the \oplea can stand random one-bit prior noise with arbitrary rate $q \in [0,1]$ when optimizing the \onemax benchmark. However, their result requires a relatively large population size of at least $\lambda = \Omega(\frac 1q n \log n)$ and then gives a runtime guarantee of $O(\frac 1q n^2 \lambda)$ fitness evaluations, significantly above the guarantee $O(\max\{n\lambda \frac{\log\log \lambda}{\log \lambda}, n \log n\})$ for the noiseless setting~\cite{DoerrK15}. 
We also note that both the required population size and the runtime guarantee contain a factor of~$\frac 1q$, that is, they become worse for lower noise rates. 
This counter-intuitive dependence on the noise intensity together with the weak runtime guarantee of order at least $n^3 \log n$ suggest that this problem is not yet fully understood, and this is why we undertake a new attempt to analyze how the \oplea solves the \onemax problem in the presence of noise.

As a main result, we prove that the \oplea with any population size $\lambda \ge C \ln(n)$, $C$ a suitable constant, can optimize the \onemax problem in the presence of bit-wise prior noise with up to constant noise probability per iteration, in asymptotically the same time as when no noise is present. We note that we regard a different noise model than in~\cite{GiessenK16}, namely independent bit-wise prior noise. This model is generally regarded as more realistic because the true and the noisy fitness can deviate by arbitrary amounts. There is no reason to believe that the \oplea should suffer less from noise in this noise model. 
In fact, we are convinced that results analogous to ours hold in the one-bit noise model regarded in~\cite{GiessenK16}, and that such results can be proven with a variant of our general proof method. 
We do not conduct these proofs since, as often in runtime analyses, the precise proofs rely on the details of the particular algorithm, objective function, and noise model. 
For that reason, even though our general approach seems to apply also to one-bit noise, the proofs would differ in many small details. We therefore leave this task for future work. 

This result is the first tight analysis of a standard population-based EA in the presence of noise in a standard model (in fact, it is the first analysis of the \oplea in the presence of bit-wise noise). 
Our result shows that already moderate population sizes can yield an enormous robustness. 
We note that for $\lambda = \Theta(\log n)$, our runtime guarantee is $O(n \log n)$, that is, the same as for the simple \oea in the noiseless setting. Hence the larger population size used here to obtain robustness does not lead to an increase in the runtime. 
We also recall the general lower bound of $\Omega(n \log n)$ valid for all unary unbiased black-box algorithm~\cite{LehreW12}, which shows that our $O(n \log n)$ bound is asymptotically tight and that a better performance is not possible in the realm of mutation-based unbiased evolutionary algorithms. 

We extend our analysis to the non-elitist \oclea and show the same results for this algorithm. In the noisefree setting, the \oplea and the \oclea are known to have a similar performance when the population size is at least logarithmic, basically because the \oclea becomes near-elitist since with high probability at least one offspring is equal to the parent. That this similarity of the two algorithms extends to noisy settings was, a priori, not obvious. We note that this is the first runtime analysis of the \oclea in the presence of prior noise. 

Our results are based on a novel proof argument, namely that the noiseless offspring can be seen as a biased crossover between the parent and the noisy offspring. This allows to obtain probabilistic information on the offspring given the parent and the noisy offspring. Note that the noisy offspring is visible to the algorithm (and thus in a sense also to our proofs) via its fitness, whereas the non-noisy offspring is not visible, but is of course what really counts for the further run of the algorithm. We refer to Sections~\ref{sec:technique} and~\ref{sec:rta} for more details. 

Overall, our work gives two main insights, namely (1)~that using EAs with at least a moderate population size can lead to a strong robustness against noise, and this without performance losses even when compared with using the optimal population size in the noiseless setting, and (2)~that such processes can be analyzed with mathematical means, in particular, with the tools we developed to gain from the noisy offspring probabilistic information on the noiseless offspring.

We complement our theoretical analysis with a small experimental study, aiming at answering two questions which our asymptotic runtime analysis naturally could not answer. We observe that, as predicted by theoretical considerations in the noisefree case, also in the noisy setting there is no significant performance difference between the \oplea and the \oclea. Also, we show that the asymptotic runtime advantage of the \oplea over the \oea at constant noise rates is clearly visible already for moderate population sizes.

This work is organized as follows. In Section~\ref{sec:related}, we review the most relevant previous works. We describe the benchmark, noise model, and algorithms in Section~\ref{sec:preliminaries}. Our main technical tool, the analysis of the relation between parent, true offspring, and noisy offspring, is developed in Section~\ref{sec:technique}. We use this tool in Section~\ref{sec:rta} to conduct the runtime analyses leading to the main results of this work. Our experimental study can be found in Section~\ref{sec:experiments}. The paper ends with a short conclusion and outlook. 

\section{Related Works}
\label{sec:related}

One of the most recent and detailed overviews of mathematical analyses of evolutionary algorithms in noisy environments can be found in the recent paper~\cite{LehreQ24}. We refer the reader to this work and discuss now only the results most relevant to ours.

\textbf{For the \oea}, several results show that it can tolerate a small amount of noise, but becomes highly inefficient for larger noise rates. In~\cite{GiessenK16}, it was shown that on \onemax with either prior bit-wise noise flipping each bit with probability $\frac{q}{n}$ or one-bit noise flipping exactly one random bit with probability $q$, the \oea keeps its $O(n\log(n))$ runtime from the noiseless setting when $q = O(\frac{1}{n})$. For $q = O(\frac{\log(n)}{n})$, the runtime is still polynomial, but for $q = \omega(\frac{\log(n)}{n})$ the runtime is super-polynomial. That is, the \oea can cope with noise only if it happens at most, on average, every $\Theta(\frac{n}{\log(n)})$ iterations. A similar effect was observed on \leadingones and made very precise in~\cite{Sudholt21}, where for the same noise models for all $q \le \frac{1}{2}$ a runtime of order $n^2 \exp(\Theta(\min\{qn^2, n\}))$ was shown. Hence here already from $q = 1/n^2$ on the performance drops rapidly with increasing noise rate.

Posterior additive Gaussian noise was studied in~\cite{GiessenK16,QianYTJYZ18}. In~\cite{GiessenK16} the authors showed that the \oea can solve \onemax and \leadingones in its noiseless time when the variance $\sigma^2$ of the Gaussian distribution is at most $\frac{1}{4\log(n)}$ and $\frac{1}{12en^2}$, respectively. In~\cite{QianYTJYZ18} it was shown that for large variance, namely $1$ and $n^2$, respectively, the runtime on these problems is exponential.

A more general approach based on  estimating the probability that the noise inverts the comparison of two individuals was used in~\cite{Dang-NhuDDIN18}. This led to a polynomial runtime guarantee on \onemax for every noise model for which the probability to invert a comparison via noise is at most $\frac{c\ln(n)}{n}$ for some constant $c$. Also, this approach allowed to make the results of~\cite{GiessenK16} more precise, among others, giving that the $O(n \log(n))$ runtime bound remains valid up to $q = c' \log(\log(n)) / n$ with $c'$ a constant specified in~\cite{Dang-NhuDDIN18}.

It is also worth mentioning the papers~\cite{QianYTJYZ18, QianBJT19, QianBYTY21}, which showed that resampling an individual sufficiently often can help to determine its true fitness. This can make the \oea robust to noise, but often requires large (more than $n^3$) numbers of samples. This strategy, however, does not work on \leadingones with high-rate prior noise. Additionally, in~\cite{DoerrS19} it was shown that using a median instead of the mean of the samples can significantly reduce the number of resamplings necessary to yield a reasonable performance, e.g.,  to $\Theta(\log(n))$ if the noise is not too strong.

\textbf{For population-based EAs}  there are considerably fewer results. In~\cite{GiessenK16} it was shown that for the \oplea and the \mpoea, the population sizes $\mu = \Omega(\frac{\log(n)}{q})$ and $\lambda = \Omega(\frac{1}{q}n\log(n))$ can lead to polynomial runtimes on \onemax with one-bit noise occurring with probability~$q$. The runtimes in this case are $O(\mu n \log(n))$ for the \mpoea, mildly above its $O(\mu n)$ noisefree runtime guarantee~\cite{Witt06}, and $O(\frac{1}{q}n^2\lambda)$ for the \oplea, significantly above its noisefree runtime guarantee of $O(n \lambda \frac{\log\log\log \lambda}{\log\log \lambda})$~\cite{DoerrK15}. This result definitely suggests that the larger populations, in particular, larger parent populations, can be beneficial to cope with noise. However, these results counter-intuitively require larger population sizes to cope with smaller noise rates. 

Another result for population-based EAs was obtained in~\cite{QianBYTY21}, where the authors considered a symmetric noise model. The authors showed that with logarithmic population sizes both the \oplea and the \mpoea have an $\tilde O(n)$ runtime on \onemax, and also that smaller population sizes for the \oplea lead to exponential runtimes. The downside of this result is that the symmetric noise model is quite artificial. In particular, for the \mpoea it implies that if we have a single individual with  best fitness, it is removed from the population only if the noise affects all  individuals in the population (which is very unlikely when $\mu$ is at least logarithmic). For the \oplea it implies that we lose fitness only when all $\Theta(\lambda)$ copies of the current individual are affected by the noise, which is also unlikely for logarithmic or larger values of $\lambda$. With these particular properties, these results appear hard to extend to the more common noise models.

In~\cite{Dang-NhuDDIN18} the authors write that the \oplea can optimize \onemax under a one-bit noise in $O(n\log(n))$ runtime, if $\lambda = \Theta(\log(np))$, where $p$ is the probability that the noise flips a bit. However, this result is only described informally, there is no theorem stating it, nor a proof or a proof sketch. 

\textbf{The study of non-elitist algorithms} in the presence of noise is confined to the three works~\cite{DangL15foga,LehreQ24,LehreQ23}. In~\cite{DangL15foga}, it is shown that a non-elitist EA constructing the next population by $\lambda$ times independently selecting two parents and taking the mutant of the better one into the next population, is very robust to prior noise with up to constant noise rates, keeping its noisefree runtime apart from a $\log\log n$ factor. Besides many other results, this result was sharpened and extended to more general noise models in~\cite{LehreQ24}, however, to the best of our understanding only for mutation rate below $1/n$. 
That paper, as well as the subsequent work~\cite{LehreQ23}, also give results for the symmetric noise model, but as discussed above, we are not optimistic that such results can be extended to more realistic noise models. 

\textbf{Other slightly less relevant results} include studies of ant colony optimizers (ACO)~\cite{FriedrichKKS16}, estimation of distribution algorithms such as the cGA~\cite{FriedrichKKS17} and the UMDA~\cite{LehreN21}, and voting algorithms~\cite{RoweA19} in the presence of noise. For an ACO with a fitness-proportional update it was shown in~\cite{FriedrichKKS16} that with a small evaporation factor $\rho=O(\frac{1}{n^3\log(n)})$ this algorithm can solve \onemax in $O(n^2\log(n)/\rho)$ time for both Gaussian or 1-bit noise with any noise rate (note that this runtime is of order larger than $n^5$). For the cGA on \onemax with Gaussian noise with variance $\sigma^2$ it was shown that with population size $K = \omega(\sigma\sqrt{n}\log(n))$ the optimum is found in $O(K\sigma^2\sqrt{n}\log(Kn))$ time~\cite{FriedrichKKS17}. With an automated choice of the algorithm parameter~$K$, the runtime stemming from the optimal choice of~$K$ can be obtained~\cite{ZhengD23}, and this without knowing in advance the noise variance~$\sigma^2$. In~\cite{LehreN21}, the UMDA was studied on \leadingones with a one-bit noise with rate less than one. It was shown that with parameters $\mu$ and $\lambda$ which are at least logarithmic in $n$, the runtime is $O(n\lambda\log(\lambda) + n^2)$. The voting algorithm was studied on \onemax with Gaussian noise, and it was shown that if the noise rate $\sigma^2$ is at least $\frac{3n}{8}$, then the optimum is correctly found after $O(\sigma^2 \log(n))$ samples.

Frozen noise models, where noise is applied to the objective function once before the algorithm is run, were studied in~\cite{FriedrichKNR22,JorritsmaLS23}. The main result is that RLS, the $(1 + 1)$ EA, and the $(1 + \lambda)$ EA are not efficient on those rugged landscapes, but the compact genetic algorithm and the $(1, \lambda)$ EA are.

Noisy optimization has also been studied for multi-objective EAs~\cite{DinotDHW23,DangOSS23gecco}, but the very different population dynamics render it difficult to compare these results with the single-objective setting.

\section{Preliminaries}\label{sec:preliminaries}

\subsection{OneMax with Bit-wise Noise}

\onemax is one of the most common benchmark functions used in theoretical studies of EAs. It returns the number of one-bits in its argument, which is formally defined by 
\begin{align*}
    \onemax(x) = \sum_{i = 1}^n x_i
\end{align*}
for all $x \in \{0, 1\}^n$.
Despite its simplicity, this benchmark has had an enormous impact on the field. It was the basis of the first mathematical runtime analyses at all~\cite{Muhlenbein92,Rudolph97,DrosteJW02}, the first analyses of EAs in the presence of noise and dynamic changes~\cite{Droste02,Droste04}, the first analyses of population-based algorithms~\cite{JansenJW05, Witt06, AntipovD21algo}, or the first analyses of estimation-of-distribution algorithms and ant-colony optimizers~\cite{Droste05,Gutjahr08,NeumannW09}. These works triggered the development of many analysis methods that are used regularly since then. The study of how EAs optimize \onemax also led to the discovery of many new algorithmic ideas, e.g., various ways to dynamically adjust the parameters of an EA~\cite{LassigS11,DoerrDE15,DoerrWY21}.

In this paper we study the optimization of \onemax under \emph{bit-wise prior noise}. In this noise model, we do not learn the correct fitness of~$x$, but the fitness of some bit string obtained from $x$ by flipping each bit independently with probability $\frac{q}{n}$. We call $\frac{q}{n}$ the \emph{noise rate}. In this paper we consider $q = O(1)$, so that the noise rate is at most $O(\frac{1}{n})$, which means that we might have a constant probability that the noise occurs and we learn a possibly wrong fitness value.

\subsection{The \oplea and the \oclea}
\label{sec:algorithms}

The two most simple EAs with a non-trivial offspring population are the \oplea and \oclea. Both algorithms work according to the following scheme. They keep a bit-string $x$, which is called the \emph{current individual} and which is initialized with a random bit string. Then in each iteration they create $\lambda$ individuals, independently, by applying standard bit mutation with rate $\frac{\chi}{n}$ to the current individual $x$ (that is, each bit in a copy of $x$ is flipped with probability $\frac{\chi}{n}$ independently from other bits). A best of these offspring is selected as the mutation winner $y$. Then the only difference in two algorithm occurs. The \oplea compares $y$ with $x$ and if $y$ is not worse, then it replaces $x$ in the next iteration. Otherwise $x$ stays the same. In the \oclea we use a non-elitist selection, and $y$ always replaces $x$, even if it is worse. In this paper we use the standard assumption that $\chi = \Theta(1)$, since smaller mutation rates usually reduce the optimization speed and larger rates can lead to an exponential runtime even on monotone functions~\cite{DoerrJSWZ13}. The pseudocodes of the \oclea and the \oplea are shown in Algorithm~\ref{alg:pseudo}, where the only difference between them is in lines~\ref{line:oclea} and~\ref{line:oplea}.

In the context of noisy optimization, it is important to note that for the \oplea we re-compute the fitness of the current individual in each iteration and do not reuse the value computed previously. This allows to correct the situation that we have a current individual for which we believe its fitness to be significantly higher (due to a noisy fitness evaluation in the past). In such a situation the EA may get stuck for a long time, simply because all offspring appear to be worse than this parent. That this problem is real and can be seen, e.g., by comparing the results of~\cite{SudholtT12} (without reevaluations) and~\cite{DoerrHK12ants} (with reevaluations).

To take into account the specifics of noisy optimization, we use the following point of view and notation when we consider an iteration of the \oplea or the \oclea. We denote the $\lambda$ offspring of $x$ created via standard bit mutation by $x^{(i)}$ (for $i \in [1..\lambda]$). Then when we apply noise to each of $\lambda$ individuals independently, we obtain $\lambda$ \emph{noisy} individuals $\tilde x^{(i)}$, where for each $i \in [1..\lambda]$ individual $\tilde x^{(i)}$ is obtained from $x^{(i)}$. When we talk about an arbitrary offspring and its noisy version, we denote them by $x'$ and $\tilde x'$ correspondingly. 
After creating noisy offspring we evaluate their \onemax values and we choose an individual $\tilde y$ from the $\lambda$ noisy ones with the best fitness value (the ties are broken uniformly at random) as the mutation winner and we denote its non-noisy parent by $y$. Then for the \oplea we also apply noise to $x$ before we compare it with $\tilde y$, and thus we get a noisy individual $\tilde x$ which competes with $\tilde y$.  

\begin{algorithm}[t]%
\caption{The \oclea and the \oplea maximizing a function $f:\{0,1\}^n \rightarrow \R$.}\label{alg:pseudo}
	\textbf{Initialization:} 
	Sample $x \in \{0,1\}^{n}$ uniformly at random and evaluate $f(x)$\;
    \textbf{Optimization:}
    \For{$t=1,2,3,\ldots$}{
		\For{$i=1,\ldots,\lambda$}{
			\label{line:k}
			$x^{(i)} \gets$ copy of $x$\;
                Flip each bit in $x^{(i)}$ with probability $\frac{\chi}{n}$\;
			evaluate $f(x^{(i)})$\;
		}
		$y \leftarrow \arg\max\{f(x^{(i)}) \mid i \in [\lambda]\}$\; 	
        \lIf{running \oclea}{$x \leftarrow y$\label{line:oclea}}
		\lIf{running \oplea and $f(y)\geq f(x)$}{$x \leftarrow y$\label{line:oplea}}
	}
\end{algorithm}

\subsection{Drift Analysis}

Drift analysis is a rapidly developing set of tools which are widely used in the analysis of random search heuristics. They help to transform easy-to-obtain information about the expected progress into bounds on the expected runtime. In this paper we use the following variable drift theorem, which was first introduced in~\cite{MitavskiyRC09, Johannsen10}. We use its simplified version for processes over integer values from~\cite{DoerrDY20}.

\begin{theorem}[Theorem 6 in~\cite{DoerrDY20}]
\label{theorem:variable_drift}
    Let $(X_t)_{t \in \N}$ be a sequence of random variables in $[0..n]$ and let T be the random variable that denotes the earliest point in time $t \ge 0$ such that $X_t = 0$. Suppose that there exists a monotonically increasing function $h: [1..n] \mapsto \R_0^+$ such that
    \begin{align*}
        E[X_t - X_{t + 1} \mid X_t] \ge h(X_t)
    \end{align*}
    holds for all $t < T$ (as an inequality of random variables). Then 
    \begin{align*}
        E[T \mid X_0] \le \sum_{i = 1}^{X_0} \frac{1}{h(i)}.
    \end{align*}
\end{theorem}

\subsection{Auxiliary Tools}

The following lemmas introduce several tools which will be useful in our proofs.

\begin{lemma}[Lemma 1.4.9 in~\cite{Doerr20bookchapter}]
\label{lem:binom-bound}
    For all $n \in \N$ and $k \in [1..n]$, we have
    $$\binom{n}{k} \geq \left(\frac{n}{k}\right)^k.$$
\end{lemma}

\begin{lemma}[Inequality 3.6.2 in~\cite{vasic2012analytic}]
\label{lem:bound_inequality}
    If $n$ and $x$ are real numbers such that $0 \leq x \leq n$ and $n > 0$, then 
    $$\left(1 - \frac{x}{n} \right)^n \geq e^{-x} - \frac{x^2}{2n}.$$
\end{lemma}

\begin{lemma}[Lemma 1 in~\cite{AntipovDK22}]
\label{lem:bernoulli}
	For all $x \in \left[0, 1\right]$ and $\lambda > 0$ we have
	\begin{align*}
		1 - (1 - x)^\lambda \ge \frac{1}{2}\min\left\{1, \lambda x\right\}.
	\end{align*}
\end{lemma}

\section{Offspring Distribution}\label{sec:technique}

In this section we discuss the relationship between the parent individual $x$, its offspring $x'$ obtained via standard bit mutation, and the individual $\tilde x'$ obtained from the offspring via bit-wise noise. 

We start with the distribution of $\tilde x'$, when we have no information about the true offspring $x'$. Each bit of the noisy offspring $\tilde x'$ is different from the bit in the same position in $x$, if and only if either it was flipped by the mutation, but not by the noise, or it was flipped by the noise, but not by the mutation. The probability of this event is
\begin{align*}
    \frac{\chi}{n}\left(1 - \frac{q}{n}\right) + \frac{q}{n}\left(1 - \frac{\chi}{n}\right) = \frac{\chi}{n} + \frac{q}{n} - 2 \cdot \frac{\chi q}{n^2} = \frac{1}{n} \left(\chi + q - \frac{2q\chi}{n}\right).
\end{align*}
Since this event is independent for all bits, $\tilde x'$ is distributed in the same way as if it was created from $x$ via standard bit mutation with rate $\frac{r}{n}$, where $r = \chi + q - \frac{2q\chi}{n}$. Since in this paper we assume $\chi = \Theta(1)$ and $q = O(1)$, we have $r = \Theta(1)$ as well.

When we know the parent $x$ and also the noisy offspring $\tilde x'$, we can estimate the distribution of the noiseless offspring $x'$, which is an intermediate step when we get $\tilde x'$ from $x$. The following lemma shows how the bit values in $x'$ are distributed when we know those bit values in $x$ and $\tilde x'$. Since, naturally, these estimates depend not on the particular bit value in $x$, but only on whether the corresponding bits agree or differ with this value, we also formulate our result in this symmetric fashion.

\begin{lemma}\label{lem:probs}
    Consider an arbitrary bit position $i$. The distribution of this bit value $x'_i$ in $x'$, conditional on the values of this bit in $x$ and $\tilde x'$, are as shown in Table~\ref{tab:probs}.
\end{lemma}

\begin{table}
\caption{Distribution of a bit value in the offspring $x'$ given the observed value in the noisy offspring~$\tilde x'$. }
\label{tab:probs}
    \begin{center}
        \begin{tabular}{cccc}\toprule
             Event $A$ & Event $B$  & $\Pr[A \mid B]$ & Approximate Value 
             \\\hline
             $x'_i= x_i$     & $\tilde x'_i = x_i$          &  $\frac{\left( 1 - \frac{\chi}{n} \right) \vphantom{\big)}\left( 1 - \frac{q}{n} \right) }{\left( 1 - \frac{\chi}{n} \right) \left( 1 - \frac{q}{n} \right)  + \frac{q\chi}{n ^ 2}}$ & $1 - \frac{q\chi}{n^2} - O\left(\frac{1}{n^3}\right)$ \\[7pt]
             $x'_i = x_i$     &  $\tilde x'_i \not = x_i$      &$\frac{\left( 1 - \frac{\chi}{n} \right) \frac{q}{n} }{\left( 1 - \frac{\chi}{n} \right) \frac{q}{n}  + \frac{\chi}{n} \left( 1 - \frac{q}{n} \right)}$ & $\frac{q}{q+\chi} \left(1 \pm o(1)\right)$ \\[7pt]
             $x'_i \not = x_i$ &  $\tilde x'_i = x_i$          & $\frac{\frac{\chi q}{n^2}}{\left( 1 - \frac{\chi}{n} \right) \left( 1 - \frac{q}{n} \right)  + \frac{q\chi}{n ^ 2}}$ & $\frac{q\chi}{n^2} + O\left(\frac{1}{n^3}\right)$\\[7pt]
             $x'_i \not = x_i$ &  $\tilde x'_i \not = x_i$      & $\frac{\left( 1 - \frac{q}{n} \right) \frac{\chi}{n} }{\left( 1 - \frac{\chi}{n} \right) \frac{q}{n}  + \frac{\chi}{n} \left( 1 - \frac{q}{n} \right)}$ & $\frac{\chi}{q+\chi} \left(1 \pm o(1)\right)$  \\
             \bottomrule
        \end{tabular}
    \end{center}
\end{table}

\begin{proof}
    We start with the first row of Table~\ref{tab:probs}, that is, we compute $\Pr[x'_i = x_i \mid \tilde x'_i = x_i]$. By the definition of conditional probabilities we have
    \begin{align*}
        \Pr[x'_i = x_i \mid \tilde x'_i = x_i] = \frac{\Pr[x'_i = x_i = \tilde x'_i]}{\Pr[x_i =  \tilde x'_i]}.
    \end{align*}
    The bit in position $i$ has the same values in all three individuals only if it has not been flipped, neither by mutation, nor by noise. The probability of this event is $(1 - \frac{\chi}{n})(1 - \frac{q}{n})$. This bit is equal in the parent $x$ and in the noisy offspring $\tilde x'$ when we either have not flipped it at all, or we flipped it twice. The probability of this event is  $(1 - \frac{\chi}{n})(1 - \frac{q}{n}) + \frac{q\chi}{n^2}$. Thus, we have

    \begin{align*}
        \Pr[&x_i' = x_i \mid \tilde x'_i = x_i] = \frac{\left( 1 - \frac{\chi}{n} \right) \left( 1 - \frac{q}{n} \right) }{\left( 1 - \frac{\chi}{n} \right) \left( 1 - \frac{q}{n} \right)  + \frac{q\chi}{n ^ 2}} \\
        &= 1 - \frac{\frac{q\chi}{n^2}}{\left( 1 - \frac{\chi}{n} \right) \left( 1 - \frac{q}{n} \right)  + \frac{q\chi}{n ^ 2}} = 1 - \frac{q\chi}{n^2} \cdot \frac{1}{\left(1 - O\left(\frac{1}{n}\right)\right)} \\
        & = 1 - \frac{q\chi}{n^2}  - O\left(\frac{1}{n^3}\right).
    \end{align*}

     To prove the second row of Table~\ref{tab:probs}, we use similar arguments to show that the event $ x_i = x'_i \ne \tilde x'_i$ happens only if we flip this bit by noise, but do not flip it by mutation. The probability of this is $(1 - \frac{\chi}{n})\frac{q}{n}$. We have $x_i \ne \tilde x'_i$,when we flip the $i$-th bit only once, the probability of which is $(1 - \frac{\chi}{n})\frac{q}{n} + (1 - \frac{q}{n})\frac{\chi}{n}$. Hence, we have
     \begin{align*}
       \Pr[&x_i' = x_i \mid \tilde x'_i \not = x_i] = \frac{\Pr[x_i' = x_i \not = \tilde x'_i]}{\Pr[x_i \not =  \tilde x'_i]} 
       = \frac{\left( 1 - \frac{\chi}{n} \right) \frac{q}{n} }{\left( 1 - \frac{\chi}{n} \right) \frac{q}{n}  + \frac{\chi}{n} \left( 1 - \frac{q}{n} \right)} \\
       &= \frac{q}{q + \chi\left(\frac{1 - \frac{q}{n}}{1 - \frac{\chi}{n}}\right)} = \frac{q}{q + \chi} \cdot \frac{q + \chi}{q + \chi \pm o\left(1\right)} = \frac{q}{q+\chi} \left(1 \pm o\left(1\right)\right),
    \end{align*}
    since $\chi = \Theta(1)$ and $q = O(1)$. 
    
    By regarding complementary events, we obtain the remaining two rows of Table~\ref{tab:probs}:    
    \begin{align*}
        \Pr[&x_i' \not = x_i \mid \tilde x'_i = x_i] = 1 - \Pr[x_i' = x_i \mid \tilde x'_i = x_i] \\
        &= \frac{\frac{\chi q}{n^2}}{\left( 1 - \frac{\chi}{n} \right) \left( 1 - \frac{q}{n} \right)  + \frac{q\chi}{n ^ 2}}  = \frac{q\chi}{n^2} \left(1 + O\left(\frac{1}{n} \right) \right)
        = \frac{q\chi}{n^2} + O\left(\frac{1}{n^3}\right),\\
         \Pr[&x_i' \not = x_i \mid \tilde x'_i \not = x_i] = 1 - \Pr[x_i' = x_i \mid \tilde x'_i \not = x_i] \\
         &= \frac{\chi}{q+\chi} \left(1 \pm o\left(1\right)\right).
         \qedhere
    \end{align*}
\end{proof}

Lemma~\ref{lem:probs} can be interpreted in the way that given $x$ and $\tilde x'$, the true offspring $x'$ is, asymptotically, a biased crossover between $x$ and $\tilde x'$, taking bit values from $x$ with probability $\frac{q}{q+\chi}$ and from $\tilde x'$ otherwise. Indeed, if $x$ and $\tilde x'$ agree in a bit value, then with high probability $x'$ also has this value in this bit (the first and the third lines of Table~\ref{tab:probs}). Where $x$ and $\tilde x'$ differ, the offspring $x'$ takes the value from $x$ with probability $\frac{q}{q+\chi} \pm o(1)$, see the second line of Table~\ref{tab:probs}, and otherwise from $\tilde x'$.  

By linearity of expectation, this observation can be lifted to distances, which is what we do in the following lemma. It implies, in particular, that if we observe some progress towards some target solution with respect to the noisy offspring, then a constant fraction of this progress is real, that is, witnessed by the true offspring. This insight will be the central tool in our later analyses. We are optimistic that it will be useful in other runtime analyses of EAs in the presence of noise as well.

\begin{lemma}\label{lem:expected-non-noisy}
    Consider some arbitrary, but fixed parent $x$. Let $d(\cdot)$ denote the Hamming distance to some arbitrary point $x^*$ in the search space. Then for any bit string $z$ we have
    \begin{align*} 
    E[d(x) &- d(x') \mid \tilde x' = z] \\
    &\geq \left(d(x) - d(z) \right) \cdot \frac{\chi}{q + \chi} 
    \cdot (1 \pm o(1)) - \frac{q\chi}{n} - O\left(\frac{1}{n^2}\right).
    \end{align*}
\end{lemma}

\begin{figure}[t]
    \centering
    \begin{tikzpicture}
        \draw (0, 1) rectangle (6, 1.5);
        \draw (1.5, 1) -- (1.5, 1.5);
        \draw (3, 1) -- (3, 1.5);
        \draw (4.5, 1) -- (4.5, 1.5);
        \node at (0.75, 1.25) {$11 \dots 1$};
        \node at (2.25, 1.25) {$11 \dots 1$};
        \node at (3.75, 1.25) {$00 \dots 0$};
        \node at (5.25, 1.25) {$00 \dots 0$}; 
        \node [left] at (-0.1, 1.25) {$x:$};
    
        \draw (0, 0) rectangle (6, 0.5);
        \draw (1.5, 0) -- (1.5, 0.5);
        \draw (3, 0) -- (3, 0.5);
        \draw (4.5, 0) -- (4.5, 0.5);
        \node at (0.75, 0.25) {$11 \dots 1$};
        \node at (2.25, 0.25) {$00 \dots 0$};
        \node at (3.75, 0.25) {$11 \dots 1$};
        \node at (5.25, 0.25) {$00 \dots 0$}; 
        \node [left] at (-0.1, 0.25) {$z:$};

        \draw [dashed] (1.5, 1) -- (1.5, 0.5);
        \draw [dashed] (3, 1) -- (3, 0.5);
        \draw [dashed] (4.5, 1) -- (4.5, 0.5);

        \draw [decorate,decoration={brace,amplitude=10pt}] (0, 1.5) -- (1.5, 1.5) node [midway,yshift=17pt] {$A$};
        \draw [decorate,decoration={brace,amplitude=10pt}] (1.5, 1.5) -- (3, 1.5) node [midway,yshift=17pt] {$B$};
        \draw [decorate,decoration={brace,amplitude=10pt}] (3, 1.5) -- (4.5, 1.5) node [midway,yshift=17pt] {$C$};
        \draw [decorate,decoration={brace,amplitude=10pt}] (4.5, 1.5) -- (6, 1.5) node [midway,yshift=17pt] {$D$};
    \end{tikzpicture}
    \caption{Illustration of the four groups of bits in the proof of Lemma~\ref{lem:expected-non-noisy}.}
    \label{fig:bit-groups}
\end{figure}
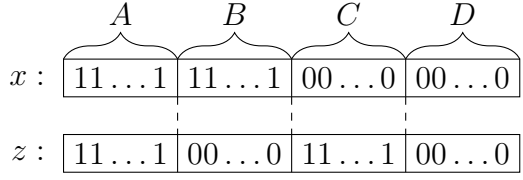

\begin{proof} 
  W.l.o.g. we assume that $x^*$ is the all-ones bit string\footnote{To generalize this proof for an arbitrary bit string $x^*$, it is sufficient to replace ``one-bits'' with ``bits which have the same value in $x^*$'' and ``zero-bits'' with ``bits which have a different value in $x^*$''. We avoid doing so to improve the readability of the proof.}, thus $d(\cdot)$ stands for the number of zero-bits in its argument. Consider some arbitrary $z$. We divide the bits into four groups depending on their values in $x$ and $z$ as illustrated in Figure~\ref{fig:bit-groups}. Let $A$ be the number of bits which are both ones in $x$ and $z$, let $B$ be the number of bits which are ones in $x$ and zeros in $z$, let $C$ be the number of bits which are zeros in $x$ and ones in $z$, and let $D$ be the number of bits with are both zeros in $x$ and $z$. Note that in each of these four groups the number of zero-bits in $x'$ follows a binomial distribution, with success probability as given in Table~\ref{tab:probs}. We denote the precise probabilities from rows $3$ and $4$ of Table~\ref{tab:probs} by $p_3$ and $p_4$ respectively (the ones in column $\Pr[A \mid B]$). Then the probabilities from rows $1$ and $2$ are $(1 - p_3)$ an $(1 - p_4)$ respectively. Hence, for any arbitrary $z$ we have

  \begin{align*}
    E&[d(x') \mid \tilde x' = z]
    = p_3 A + p_4 B + (1 - p_4) C + (1 - p_3) D \\
    &= D + C + p_3 (A - D) + p_4 (B - C) \\ 
    &= (D + C) + (A - D) \left (\frac{q\chi}{n^2} + O\left(\frac{1}{n^3}\right) \right) + (B - C)\frac{\chi}{q + \chi} (1 \pm o(1)) \\
    &\le d(x) + n\left (\frac{q\chi}{n^2} + O\left(\frac{1}{n^3}\right) \right) + \left(d(z) - d(x) \right)  \frac{\chi}{q + \chi} (1 \pm o(1)) \\
    &= d(x) + \frac{q\chi}{n} + O\left(\frac{1}{n^2}\right) + \left(d(z) - d(x) \right) \frac{\chi}{q + \chi} (1 \pm o(1)), \\
 \end{align*}
since $C + D = d(x)$, $A - D \leq n$, and $B - C = d(z) - d(x)$. By moving $d(x)$ to the left hand side and multiplying both sides by $-1$, we finish the proof.
\end{proof}

\section{Runtime Analysis}\label{sec:rta}

We use the results from Section~\ref{sec:technique} to prove upper bounds on the runtime of the \oplea and the \oclea. The main result of this section is the following theorem.

\begin{theorem}\label{thm:runtime}
    Consider a run of the \oclea or the \oplea with mutation rate $\frac{\chi}{n}$ on the \onemax problem with bit-wise noise with rate $\frac{q}{n}$. Here $\chi$ can be any positive constant (independent of the problem size~$n$). The parameter $q$ may depend on~$n$ as long as $q = O(1)$. If the population size $\lambda$ is at least $C \ln(n)$ for some constant $C$ depending on $\chi$ and $q$, then the expected number $E[T_F]$ of fitness evaluations until the optimum is sampled is
    \begin{align*}
        O\left(n\log(n) + n\lambda \frac{\log\log(\lambda)}{\log(\lambda)} \right).
    \end{align*}
\end{theorem}

To prove Theorem~\ref{thm:runtime}, we use the variable drift theorem (Theorem~\ref{theorem:variable_drift}). The process we apply it to is $\{d_t\}_{t \in \N_0}$, which is the distance of the current individual (of either the \oclea or the \oplea) to the optimum after iteration $t$. Since we consider the \onemax problem, $d_t$ is equal to the number of zero-bits in that individual.

To bound the drift of $d_t$, we study what happens in one iteration of the two considered algorithms. We use the notation defined in Section~\ref{sec:algorithms}, and additionally we denote the individuals accepted as the parent for the next iteration by $x_{\text{com}}$ for the \oclea and by $x_{\text{plus}}$ for the \oplea. Note that $x_{\text{com}}$ is always equal to $y$ (the winning offspring), while $x_{\text{plus}}$ can be either $y$ or $x$, depending on their noisy comparison. We also denote the distances to the optimum from different individuals as follows. By $d$ and $\tilde d$ we denote the distance to the optimum from the parent $x$ and the noisy parent~$\tilde x$, respectively. By $d_y$ and $\tilde d_y$ we denote the distance to the optimum from the mutation winner $y$ and its noisy version $\tilde y$, respectively. By $d_{\text{com}}$ and $d_{\text{plus}}$ we denote distances from next generation parents $x_{\text{com}}$ and $x_{\text{plus}}$ to the optimum, respectively.

With this notation the drift in one iteration is defined as
\begin{align*}
    \Delta^{\text{com}}(d) &\coloneqq E[d - d_{\text{com}}] \text{ for the \oclea, and}\\
    \Delta^{\text{plus}}(d) &\coloneqq E[d - d_{\text{plus}}] \text{ for the \oplea}.
\end{align*}
We estimate these drifts in the following lemma.

\begin{lemma}
\label{lem:drift}
    
Let
    \begin{align*}
        \Delta^+(d) &\coloneqq \sum_{i = 1}^{d} \Pr[\tilde d_y \le d - i], \text{ and} \\
        \Delta^-(d) &\coloneqq \sum_{j = 1}^{n - d} \Pr[\tilde d_y \ge d + j].
    \end{align*}
    Then the drift of the current individual $x$ towards the optimum in one iteration is at least
    \begin{align}
    \label{eq:delta-com}
        \Delta^{\text{com}}(d) &\ge (1 - o(1)) \frac{\chi}{q + \chi} \left(\Delta^+(d) - \Delta^-(d)\right) - \frac{q\chi}{n} - O\left(\frac{1}{n^2}\right)
    \end{align}
    for the \oclea, and it is at least
    \begin{align}
    \label{eq:delta-plus}
    \begin{split}
        \Delta^{\text{plus}}(d) &\ge (1 - o(1)) \frac{\chi}{q + \chi} \left((1 - o(1))e^{-q}\Delta^+(d) - \Delta^-(d)\right) 
        - \frac{q\chi}{n} - O\left(\frac{1}{n^2}\right)
    \end{split}
    \end{align}
    for the \oplea.
\end{lemma}

Before we prove this lemma we note that $\Delta^+(d)$ and $\Delta^-(d)$ are the positive and the negative components of the drift of $\tilde y$ from $x$ respectively, since we have $E[d - \tilde d_y] = \Delta^+(d) - \Delta^-(d)$.

\begin{proof}
    In this proof we consider a fixed parent individual $x$, and therefore $d$ is also fixed (that is, it is not a random variable). For the \oclea we have $x_\text{com} = y$, hence we have 
    
    \begin{align*}
        \Delta^{\text{com}}(d) &=  E[d - d_{\text{com}}] = E[d - d_y] \\
        &= \sum_{Y \in \{0, 1\}^n} \Pr[\tilde y = Y] E[d - d_y \mid \tilde y = Y]. \\
    \end{align*}
    For any bit string $z \in \{0, 1\}^n$ we denote the distance to the optimum by $d_z$. Then by Lemma~\ref{lem:expected-non-noisy} we have
    \begin{align*}
        \Delta^{\text{com}}(d) 
        &\ge \sum_{z \in \{0, 1\}^n} \Pr[\tilde y = z] \bigg(\left(d - d_z \right) \frac{\chi}{q + \chi} (1 \pm o(1))  - \frac{q\chi}{n} - O\left(\frac{1}{n^2}\right)\bigg) \\
        &= \left((1 - o(1)) \frac{\chi}{q + \chi} \sum_{z \in \{0, 1\}^n} \Pr[\tilde y = z] \left(d - d_z \right)\right) - \frac{q\chi}{n} - O\left(\frac{1}{n^2}\right) \\
        &= (1 - o(1)) \frac{\chi}{q + \chi} E[d - \tilde d_y] - \frac{q\chi}{n} - O\left(\frac{1}{n^2}\right).
    \end{align*}

    Noting that $E[d - \tilde d_y] = \Delta^+(d) - \Delta^-(d)$ by the definition of $\Delta^+(d)$ and $\Delta^-(d)$ completes the proof of eq.~\eqref{eq:delta-com}.

    Estimating drift for the \oplea requires more effort, since we are not guaranteed that $x_\text{plus} = y$. This happens if and only if $\tilde y$ appears no worse than $\tilde x$ (the individual we obtain from $x$ via noise). Consequently, we have
    \begin{align}\label{eq:drift-plus}
        d - d_{\text{plus}} = (d - d_y) \I\left[\tilde d \ge \tilde d_y\right],
    \end{align}
    where all random variables are over the joint probability space of noise on $x$, noise on offspring and the mutation, $\I[\cdot]$ is an indicator random variable, and eq.~\eqref{eq:drift-plus} is an identity of random variables functions. Therefore, by the law of total probability, the drift for the \oplea is
    \begin{align*}
        \Delta&^{\text{plus}}(d) = E[d - d_{\text{plus}}] = E\left[(d - d_y) \I\left[\tilde d \ge \tilde d_y\right]\right] \\
        &= \sum_{z \in \{0, 1\}^n} \Pr[\tilde y = z] \cdot \bigg(\Pr\left[\tilde d \ge \tilde d_y \mid \tilde y = z\right]  \\
        &\quad\,\cdot E\left[d - d_y \mid \tilde y = z, \tilde d \ge \tilde d_y\right] + \Pr\left[\tilde d < \tilde d_y \mid \tilde y = z\right] \cdot 0 \bigg).
    \end{align*}
    
    Note that when we fix $x$ and $\tilde y$, then event $\tilde d \ge \tilde d_y$ depends only on the noise which affects $\tilde x$. This noise does not affect any offspring of $x$, hence the random variable $d - d_y$ conditioned on $\tilde y = z$ is independent of $\tilde d \ge \tilde d_y$, which together with Lemma~\ref{lem:expected-non-noisy} implies that 
    \begin{align*}
        E\Big[d - d_y &\mid \tilde y = z, \tilde d \ge \tilde d_y\Big] = E\left[d - d_y \mid \tilde y = z\right] \\
        &\ge (1 \pm o(1)) \frac{\chi}{q + \chi} \left(d - d_z \right) - \frac{q\chi}{n} - O\left(\frac{1}{n^2}\right).
    \end{align*}  
    Hence, we have
    \begin{align}\label{eq:delta-plus-prelim}
        \begin{split}
            \Delta^{\text{plus}}(d) &\ge (1 - o(1)) \frac{\chi}{q + \chi} 
            \cdot\sum_{z \in \{0, 1\}^n} \Pr[\tilde y = z] \Pr[\tilde d \ge \tilde d_y \mid \tilde y = z] (d - d_z) \\
            &\quad\,- \frac{q\chi}{n} - O\left(\frac{1}{n^2}\right).
        \end{split}
    \end{align}
    
    When $d_z > d$, we estimate the conditional probability $\Pr[\tilde d \ge \tilde d_y \mid \tilde y = z]$ by one. Otherwise, if $d_z < d$, this probability is at least the probability that $\tilde x = x$, that is, that the noise has not flipped any bit in $x$ when we compared it with $\tilde y$. This probability is $(1 - \frac{q}{n})^n = (1 - o(1))e^{-q}$. With this observation, similar to the \oclea, we can rewrite the sum above as follows
    \begin{align*}
        &\sum_{z \in \{0, 1\}^n} \Pr[\tilde y = z] \Pr[\tilde d \ge \tilde d_y \mid \tilde y = z] (d - d_z) \\
        &\qquad \ge \sum_{i = 1}^{d} (1 - o(1))e^{-q} i \Pr[\tilde d_y = d - i] - \sum_{j = 1}^{n - d} j \Pr[\tilde d_y = d + j] \\
        &\qquad = (1 - o(1))e^{-q} \Delta^+(d) - \Delta^-(d).
    \end{align*}
    Putting this into eq.~\eqref{eq:delta-plus-prelim} completes the proof of eq.~\eqref{eq:delta-plus}. 
\end{proof}

Above we obtained a weaker drift estimate for the \oplea than the \oclea. This is counter-intuitive -- one would feel that the \oplea should profit to some extent from the property that the current-best solution is participating in the selection of the next parent. The reason for our weaker bound is the estimate following eq.~\eqref{eq:delta-plus-prelim}, where we pessimistically estimated that inferior mutation winners $y$ are always accepted, but superior ones only when the parent is not subject to noise. We feel that this cannot be avoided in the general case (note that Lemma~\ref{lem:drift} is valid in general and is not specific for the optimization of \onemax). We are sure that when exploiting properties of \onemax, stronger bounds could be shown. We refrain from this both because we like our general, problem-independent approach and because all we can gain are constant factors, which are generally ignored in an asymptotic analysis as ours (note that we lose constant factors and lower order terms in the later part of the analysis anyway).

In the following lemmas we show estimates for $\Delta^+(d)$ and $\Delta^-(d)$, which naturally are specific to the \onemax problem. Lemmas~\ref{lem:positive-large-d} to~\ref{lem:positive-small-d} estimate the positive component of the drift for different distances. We note that implicitly these results exist in the proofs in~\cite{DoerrK15}, but distilling them from there is rather complicated.

\begin{lemma}[Positive drift, large distance]
\label{lem:positive-large-d}
    If $\lambda=\omega(1)$, then for any distance $d \ge \frac{n}{\ln\lambda}$ we have $\Pr[d - \tilde d_y \geq \lfloor\frac{\ln \lambda}{2 \ln \ln \lambda}\rfloor] \geq \frac{1}{4}$ and the positive component of the drift is at least 
    \begin{align*}
        \Delta^+(d) \ge \frac{1}{4} \left\lfloor \frac{\ln \lambda}{2 \ln \ln \lambda}\right\rfloor.
    \end{align*}
\end{lemma}

\begin{proof}
    Let $x'$ be some arbitrary offspring and let $\tilde x'$ be the same offspring after noise. Denote the distances from them to the optimum by $d_{x'}$ and $\tilde d_{x'}$ correspondingly. Let also $p_i = \Pr[d - \tilde d_{x'} \geq i] \geq \Pr[d - \tilde d_{x'} = i] \geq \binom{d}{i} \left( \frac{r}{n} \right)^i \left(1 - \frac{r}{n} \right)^{n - i}$, since each noisy offspring is distributed as if it was created via standard bit mutation with rate $\frac{r}{n}$ where $r = \chi + q - \frac{2q\chi}{n}$. Then we compute 
    \begin{align*}
        p_{\lfloor\frac{\ln \lambda}{2 \ln \ln \lambda}\rfloor} &\geq \binom{d}{\lfloor\frac{\ln \lambda}{2 \ln \ln \lambda}\rfloor} \left( \frac{r}{n} \right)^{\lfloor\frac{\ln \lambda}{2 \ln \ln \lambda}\rfloor} \left(1 - \frac{r}{n} \right)^{n - \lfloor\frac{\ln \lambda}{2 \ln \ln \lambda}\rfloor} \\
        &\geq \left(\frac{d}{\lfloor\frac{\ln \lambda}{2 \ln \ln \lambda}\rfloor}\right)^{\lfloor\frac{\ln \lambda}{2 \ln \ln \lambda}\rfloor} \left( \frac{r}{n} \right)^{\lfloor\frac{\ln \lambda}{2 \ln \ln \lambda}\rfloor} \left(1 - \frac{r}{n} \right)^{n}\\
        &\geq \left(\frac{dr}{\lfloor\frac{\ln \lambda}{2 \ln \ln \lambda}\rfloor n}\right)^{\lfloor\frac{\ln \lambda}{2 \ln \ln \lambda}\rfloor} \left(e^{-r} - \frac{r^2}{2n}\right),
    \end{align*}
    where transition to the second line follows from Lemma~\ref{lem:binom-bound} and the next transition follows from Lemma~\ref{lem:bound_inequality}.
    When $n$ is large enough, the last term is close to $e^{-r}$, thus we can bound it from below with $e^{-2r}$. Hence, recalling that we consider the case when $d \ge \frac{n}{\ln \lambda}$, we have 
    \begin{align*}
        p_{\lfloor\frac{\ln \lambda}{2 \ln \ln \lambda}\rfloor} 
        &\geq \exp\left(\frac{\ln \lambda}{2 \ln \ln \lambda} \left( \ln {\frac{dr}{n}} - \ln {\frac{\ln \lambda}{2 \ln \ln \lambda}}\right)  - 2r\right) \\
        &\geq \exp\bigg( \frac{\ln \lambda}{2 \ln \ln \lambda} \bigg( \ln \frac{r}{\ln\lambda} - \ln \ln \lambda + \ln (2\ln \ln \lambda ) 
        - 2r \cdot \frac{2\ln\ln\lambda}{\ln\lambda}\bigg)\bigg). 
    \end{align*}
    Since $\ln(2\ln\ln(\lambda)) = \omega(1)$ and $\ln r$ and $2r\cdot \frac{2\ln\ln\lambda}{\ln\lambda}$ are both $O(1)$, when $n$ is large enough, we have   
    \begin{align*}
        p_{\lfloor\frac{\ln \lambda}{2 \ln \ln \lambda}\rfloor} \geq \exp\left(\frac{\ln \lambda}{2 \ln \ln \lambda} \cdot (-2\ln\ln\lambda) \right) = e ^{-\ln \lambda} = \frac{1}{\lambda}.
    \end{align*}
    The probability that at least one of the $\lambda$ individuals is by $\lfloor \frac{\ln \lambda}{2\ln\ln\lambda}\rfloor$ closer to the optimum than $x$ is therefore at least $(1 - \frac{1}{\lambda})^\lambda \geq \frac{1}{4}$. Hence, we have 
    \begin{align*}
        \Delta^+(d) \ge p_{\lfloor\frac{\ln \lambda}{2 \ln \ln \lambda}\rfloor} \left\lfloor\frac{\ln \lambda}{2 \ln \ln \lambda}\right\rfloor \ge  \frac{1}{4} \left\lfloor \frac{\ln \lambda}{2 \ln \ln \lambda}\right\rfloor.
    \end{align*}
\end{proof}

The next lemma covers the case when $d \in [\frac{n}{\lambda}, \frac{n}{\ln \lambda}]$.

\begin{lemma}[Positive drift, medium distance]
\label{lem:positive-medium-d}
    For $d \in [\frac{n}{\lambda}, \frac{n}{\ln \lambda}]$ we have $\Pr[d - \tilde d_{y} \geq 1] \geq \frac{r}{\lambda e^r} (1 - o(1))$ and the positive drift is at least $\Delta^+(d) \ge \frac{r}{2e^r} (1 - o(1)) = \Theta(1)$, where $r = q + \chi - \frac{2q\chi}{n}$.
\end{lemma}

\begin{proof}
    Similar to Lemma~\ref{lem:positive-large-d}, let $p_1 = \Pr[d - \tilde d_{x'} \geq 1]$. Then we have
    \begin{align*}
    p_1 \geq \frac{dr}{n} \left( 1 - \frac{r}{n} \right)^{n - 1} = \frac{dr}{n}\left(\frac{1}{e^r}-o(1)\right) \geq \frac{r}{\lambda e^r} (1-o(1))
    \end{align*}
    by Lemma~\ref{lem:bound_inequality}. The probability that the noisy mutation winner is better than the current individual is the probability that at least one of the noisy offspring is better. Hence, we have
    \begin{align*}
        \Delta^+(d) &\geq 1 \cdot \Pr\left[d - \tilde d_{y}\geq 1\right] \geq 1 - (1 - p_1) ^\lambda \\
        &\geq \frac{1}{2} \min\left(1, \lambda p_1\right) \geq \frac{r}{2e^r} (1 - o(1))
    \end{align*}
    by Lemma~\ref{lem:bernoulli} and since $\frac{r}{e^r} \leq 1$ for all $r$. Since $r=\Theta(1)$, this lower bound is also $\Theta(1)$.
\end{proof}

Finally, we estimate the positive drift when $d \leq \frac{n}{\lambda}$.

\begin{lemma}[Positive drift, small distance]
\label{lem:positive-small-d}
    For $d \leq \frac{n}{\lambda}$ we have $\Pr[d - \tilde d_{y} \geq 1] \geq \frac{r}{\lambda e^r} (1 - o(1))$ and the positive drift is at least $\Delta^+(d) \ge \frac{\lambda dr}{2ne^r}(1-o(1)) $.
\end{lemma}

\begin{proof}
    Similar to Lemma~\ref{lem:positive-large-d} and Lemma~\ref{lem:positive-medium-d}, let $p_1 = \Pr[d - \tilde d_{x'} \geq 1 ]$. Then we have
    \begin{align*}
        p_1 \geq \frac{dr}{n} \left( 1 - \frac{r}{n} \right)^{n - 1} = \frac{dr}{ne^r}(1-o(1)).
    \end{align*}
    Therefore, similar to the proof of Lemma~\ref{lem:positive-medium-d}, by Lemma~\ref{lem:bernoulli} we have
    \begin{align*}
        \Delta^+(d) &\geq \frac{1}{2} \min\left(1, \lambda p_1\right) \geq \frac{\lambda dr}{2ne^r}(1-o(1)) 
    \end{align*}
    due to $\frac{\lambda d}{n} \leq 1$ and $\frac{r}{e^r} \leq 1$.
\end{proof}

We proceed with bounding the negative component of the drift. A similar estimate for this negative part of the drift of the \oclea on the noiseless \onemax was shown in~\cite{RoweS14}. However, there this drift was bounded with a constant, which is not enough in our situation, so we give a stronger bound.

\begin{lemma}[Negative drift]
\label{lem:negative-drift}
If $\lambda > C\ln(n)$ for some sufficiently large constant $C$ which depends on $r$, then for all $d \in [1..n]$ we have $\Delta^-(d) \le  \frac{1}{n}.$
\end{lemma}

\begin{proof}
    If $\tilde d_y$ is at least~$d$, then for all $\lambda$ noisy offspring $\tilde x'$ their distances to the optimum $\tilde d'$ are at least~$d$.
    Then we have
    
    \begin{align*}
        \sum_{j=1}^{n-d}\Pr[d - \tilde d_{y} \geq j ] = \sum_{j=1}^{n-d}\left(\Pr[d - \tilde{d}' \geq j ]\right)^\lambda.
    \end{align*}
    
    Let $p_j \coloneqq \Pr[d - \tilde{d}' \geq j ]$. Since $p_1$ is at most the probability that at least one 1-bit in $x$ is flipped, we have  $p_1 \leq 1 - \left(1 - \frac{r}{n}\right)^{n-d} \le 1 - e^{-r} + o(1)$ (the last inequality is by Lemma~\ref{lem:bound_inequality}). Hence,
    \begin{align*}
        p_1^\lambda \leq \left(1-e^{-r} + o(1)\right)^\lambda \leq \frac{1}{n^2}
    \end{align*}
    for $\lambda \geq C \ln{n}$ with $C = -\frac{2}{\ln (1 - e^{-r} + o(1))}$. 
    For all $j \in \N$ we have $p_j \leq p_1$, and therefore $p_j^\lambda \leq \frac{1}{n^2}$. Consequently,
    \begin{align*}
        \Delta^-(d) = \sum_{j=1}^{n-d}\left(\Pr[d - \tilde{d}' \geq j ]\right)^\lambda=\sum_{j=1}^{n-d}p_j^\lambda \leq \frac{n-d}{n^2} \leq \frac{1}{n}. &\qedhere
    \end{align*}
\end{proof}

With Lemma~\ref{lem:drift} and with estimates given in Lemmas~\ref{lem:positive-large-d} to~\ref{lem:negative-drift} we are now in the position to prove our main result, Theorem~\ref{thm:runtime}. 

\begin{proof}[Proof of Theorem~\ref{thm:runtime}]
    Let $h^+(d)$ be the lower bound on $\Delta^+_t(d)$ derived in Lemmas~\ref{lem:positive-large-d} to~\ref{lem:positive-small-d}, that is,
    \begin{align}\label{eq:h-plus}
        \Delta^+(d) &\ge  h^+(d) \coloneqq 
        \begin{cases}
            \frac{1}{4} \left\lfloor \frac{\ln \lambda}{2\ln\ln\lambda} \right\rfloor &\text{ if } d > \frac{n}{\ln n}, \\
            \frac{r}{2e^r} (1 - o(1)) &\text{ if } d \in \left] \frac{n}{\lambda}, \frac{n}{\ln \lambda} \right], \\
            \frac{\lambda d r}{2 n e^r} (1 - o(1)) &\text{ if } d \leq \frac{n}{\lambda}.    
        \end{cases}
    \end{align}
    Note that this is a monotonically increasing function in $d$.
    Let also $h^-(d)$ be the upper bound on $\Delta^-(d)$ from Lemma~\ref{lem:negative-drift}, that is,
    \begin{align*}
        \Delta^-(d) \le  h^-(d) \coloneqq \frac{1}{n},
    \end{align*}
    for which we require $\lambda$ to be at least $C\ln(n)$, where $C$ is the constant from Lemma~\ref{lem:negative-drift}.

    Using Lemma~\ref{lem:drift}, we define lower bounds on the drift of two algorithms, that are,

    \begin{align*}
        \Delta^{\text{com}}(d) &\ge h^{\text{com}}(d)  \coloneqq (1 - o(1)) \frac{\chi}{q + \chi} \left(h^+(d) - h^-(d)\right) - \frac{q\chi}{n} - O\left(\frac{1}{n^2}\right), \\
        \Delta^{\text{plus}}(d) &\ge h^{\text{plus}}(d)  \coloneqq (1 - o(1)) \frac{\chi}{q + \chi} \left((1 - o(1))e^{-q}h^+(d) - h^-(d)\right) - \frac{q\chi}{n} - O\left(\frac{1}{n^2}\right).
    \end{align*}

    We now aim at showing that the first term (the one containing $h^+(d)$) is dominant in both $h^{\text{com}}(d)$ and $h^{\text{plus}}(d)$ for all values of $d$. For this we note that by Lemmas~\ref{lem:positive-large-d} to~\ref{lem:positive-small-d}, we have $h^+(d) = \omega(\frac{1}{n})$. By Lemma~\ref{lem:negative-drift} we have $h^-(d) = O(\frac{1}{n}) = o(h^+(d))$. By the constraints $\chi = \Theta(1)$ and $q = O(1)$ we also have $\frac{\chi}{q + \chi} = \Theta(1)$ and $\frac{q\chi}{n} = O(\frac{1}{n}) = o(h^+(d))$, and $e^{-q} = \Theta(1)$. Therefore, the expressions for the lower bounds on the drift are simplified as follows.
    \begin{align*}
        h^{\text{com}}(d) &= (1 - o(1)) \frac{\chi}{q + \chi} h^+(d) \\
        h^{\text{plus}}(d) &= (1 - o(1)) \frac{\chi e^{-q}}{q + \chi} h^+(d).
    \end{align*}
    
    Let $T^{\text{com}}_I$ and $T^{\text{plus}}_I$ be the number of iterations until the \oclea and the \oplea (respectively) find the optimal solution and accept it as $x$ for the first time, that is, the minimum $t$ when we have $d_t = 0$. Then by the variable drift theorem (Theorem~\ref{theorem:variable_drift}) we have
    \begin{align*}
        E[T^{\text{com}}_I]  &\le \sum_{d = 1}^n \frac{1}{h^\text{com}(d)} = \frac{(1 + o(1))(q + \chi)}{\chi} \sum_{d = 1}^n \frac{1}{h^+(d)},\\
        E[T^{\text{plus}}_I]  &\le \sum_{d = 1}^n \frac{1}{h^\text{plus}(d)} = \frac{(1 + o(1))(q + \chi)e^q}{\chi} \sum_{d = 1}^n \frac{1}{h^+(d)}.
    \end{align*}

    By eq.~\eqref{eq:h-plus}, we have
    \begin{align*}
        \sum_{d = 1}^n \frac{1}{h^+(d)} &\leq \sum_{d = \lfloor \frac{n}{\ln\lambda} \rfloor + 1}^n \frac{8}{\frac{\ln\lambda}{\ln\ln\lambda} - 2} + \sum_{d = \lfloor \frac{n}{\lambda} \rfloor + 1}^{\lfloor \frac{n}{\ln\lambda} \rfloor} \frac{2e^r}{r} + \sum_{d = 1}^{\lfloor \frac{n}{\lambda} \rfloor} \frac{2ne^r}{\lambda dr} \\
        &\le n \cdot \frac{9\ln\ln\lambda}{\ln\lambda} + \frac{n}{\ln\lambda} \cdot \frac{2e^r}{r} + \frac{2ne^r}{\lambda r} \left(\ln\frac{n}{\lambda} + 1\right) \\
        &= O\left( \frac{n \log\log\lambda}{\log \lambda} \right) + O\left( \frac{n}{\log\lambda} \right) + O\left(\frac{n}{\lambda} \log \frac{n}{\lambda}\right) \\
        &=O\left(\frac{n}{\lambda} \left(\log \frac{n}{\lambda} + \frac{\lambda \log\log\lambda}{\log \lambda}\right)\right).
    \end{align*}
    Hence, we have
        \begin{align*}
        E[T^{\text{plus}}_I], E[T^{\text{com}}_I] &=O\left(\frac{n}{\lambda} \left(\log \frac{n}{\lambda} + \frac{\lambda \log\log\lambda}{\log \lambda}\right)\right).
    \end{align*}

    Since  each iteration of the \oclea consists of $\lambda$ fitness evaluations, and each iteration of the \oplea consists of $\lambda + 1 = \Theta(\lambda)$ fitness evaluations, the expected number of fitness evaluations until we find the optimum is at most
    \begin{align*}
        E[T^{\text{plus}}_F], E[T^{\text{com}}_F] = O\left(n\log\frac{n}{\lambda} + n \frac{\lambda\log\log\lambda}{\log\lambda}\right). &\qedhere
    \end{align*}
\end{proof}

We note that for all $\lambda = O(\frac{\log(n)\log\log(n)}{\log\log\log(n)})$ the first term of the upper bound in Theorem~\ref{thm:runtime} is dominating, hence the expected runtime is $O(n\log(n))$, that is, the same as in the noiseless case.

\section{Experiments}\label{sec:experiments}

In this section we describe the results of our empirical study, the goal of which is to give a better understanding of the noisy optimization process and to answer some questions for which our theoretical analysis was not detailed enough.

When we estimated the drift of the \oplea in Lemma~\ref{lem:drift}, we pessimistically assumed that the elitism of this algorithm can be only harmful. 
Namely, in our proofs, the comparison with the parent does not save us from decreasing the fitness (since the parent's fitness might be decreased by the noise) and it might prevent us from increasing fitness (since the parent's fitness might be increased by the noise). 
This pessimistic view eased the proof without harming the asymptotic order of magnitude of the runtime, so it was fully appropriate for the theoretical analysis. Still, it raises the question which of the two algorithms is better when looking at the precise runtime rather than its asymptotics. 
To understand this question, we ran the \oplea and the \oclea on \onemax with problem sizes $n = \{2^6, 2^7,\dots, 2^{14}\}$, with $100$ repetitions for each value of~$n$. We used a strong noise with $q=1$ to maximize its effect, standard bit mutation with $\chi = 1$, and we chose $\lambda = \lceil C \ln(n)\rceil$ with $C = 14$, since for the chosen values of $q$ and $\chi$ this $C$ would satisfy the assumptions of Lemma~\ref{lem:negative-drift}. 
The mean runtimes of these runs and their standard deviations are shown in the plot in Figure~\ref{plot:oplea_vs_oclea}. For better visual comparison, we normalize both runtimes by $n\ln(n)$, which is our upper bound on their asymptotical runtime and the well-known lower bound for all unary unbiased black-box algorithms~\cite{LehreW12}.
In the plot we see that for all problem sizes the difference between the mean runtimes of two algorithms is very small compared to the standard deviations of their runtimes. 
Therefore, we conclude that in practice for this population size the elitism does not slow down the algorithm by more than some lower order terms (but neither can speed it up). We note that this is well-known for the noisefree setting, simply because with high probability some offspring equals the parent, so removing the parent in the \oclea has no negative effect. For the noisy setting, as our proof shows, this is less obvious.

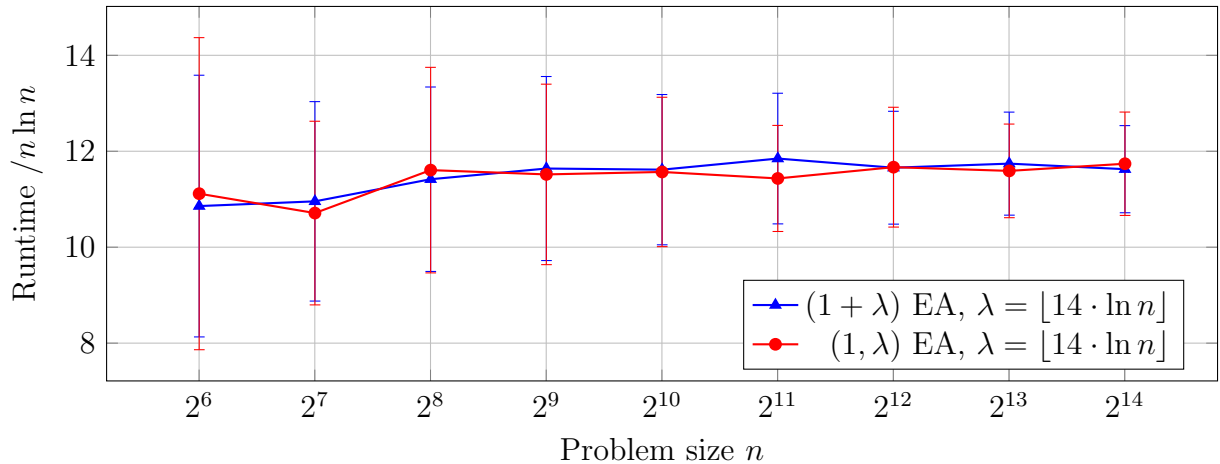
\begin{figure}[t]
\begin{tikzpicture}
    \begin{axis}[width=\linewidth, height=0.3\textheight,
                  anchor=outer north east,
                  legend pos = south east,
                  legend cell align={right},
                  cycle list name=runtimes, grid=major, xmode=log, log base x=2,
                  xlabel={Problem size $n$}, ylabel={Runtime $/n \ln n$}, title={}]
        		\addplot plot [error bars/.cd, y dir=both, y explicit] coordinates
		{(64,10.857106726723293)+-(0,2.727455603945208)(128,10.955594278807789)+-(0,2.0786901765560963)(256,11.416761570547305)+-(0,1.9224262858565924)(512,11.639668599249937)+-(0,1.9181262030346706)(1024,11.616002827668826)+-(0,1.5663573195843616)(2048,11.847607894601424)+-(0,1.3616279564956961)(4096,11.656842967155791)+-(0,1.1761826317525652)(8192,11.74182299649935)+-(0,1.0739811706868558)(16384,11.626008203579145)+-(0,0.9082320242441468)};
		\addlegendentry{\oplea, $\lambda=\lfloor14 \cdot \ln n\rfloor$};
		\addplot plot [error bars/.cd, y dir=both, y explicit] coordinates
		{(64,11.115439307482475)+-(0,3.2526591093743624)(128,10.711414922835901)+-(0,1.9141526050473219)(256,11.606692936233092)+-(0,2.1423027784018074)(512,11.517205455750869)+-(0,1.8810427903196965)(1024,11.568399526846997)+-(0,1.5571903581323732)(2048,11.432972677731732)+-(0,1.105550441937512)(4096,11.668106097202187)+-(0,1.2488644411322969)(8192,11.590666904895091)+-(0,0.9759516989811412)(16384,11.739484644479061)+-(0,1.0773332501056068)};
		\addlegendentry{\oclea, $\lambda=\lfloor14 \cdot \ln n\rfloor$};
    \end{axis}
\end{tikzpicture}
\caption{Mean runtimes (number of fitness evaluations, normalized by $n \ln n$) and their standard deviations over $100$ runs of the \oplea and the \oclea on \onemax with noise rates $q = 1$ with varying problem size $n$.}
\label{plot:oplea_vs_oclea}
\end{figure}

We also compare the \oplea with $\lambda = \lfloor\ln(n)\rfloor$ and the \oea on \onemax in the presence of bit-wise noise with small but constant value of $q=0.01$. Ours and the previous theoretical works showed that in this setting the asymptotic runtime of the \oplea is much better, but the question is if this translates to significant runtime differences already for common problem sizes. Our results, depicted in Figure~\ref{plot:oplea_vs_oea_w_reevaluations}, show that the difference between the algorithms is notable already on the smallest problem size $n = 2^6$. We also note that since we use a logarithmic scale for both axes, the convex plot of the \oea indicates that its runtime is super-polynomial in problem size $n$, which lines up with the theoretical results in~\cite{GiessenK16}.

\begin{figure}[!h]
\begin{tikzpicture}
    \begin{loglogaxis}[width=\linewidth, height=0.3\textheight,
                  legend pos = north west,
                  legend cell align={left},
                  cycle list name=runtimes, grid=major, 
                  log base x=2,
                  log base y=10,
                  xlabel={Problem size $n$}, ylabel={Runtime}, title={}]

        		\addplot plot [error bars/.cd, y dir=both, y explicit] coordinates
		{(64,859.7)+-(0,292.05891528936417)(128,2016.15)+-(0,576.7390462765635)(256,4490.16)+-(0,1122.369410844754)(512,9926.84)+-(0,2043.907051311287)(1024,22549.59)+-(0,5899.628376592885)};
		\addlegendentry{\oplea, $\lambda=\lfloor\ln(n)\rfloor$};
		\addplot plot [error bars/.cd, y dir=both, y explicit] coordinates
		{(64,1405.88)+-(0,592.1409845636426)(128,3852.68)+-(0,1757.7260815041686)(256,13421.7)+-(0,8039.610012307811)(512,125442.64)+-(0,101046.55746847784)(1024,12002447.42)+-(0,11927519.757871604)};
		\addlegendentry{\oea};
    \end{loglogaxis}

\end{tikzpicture}
\caption{Mean runtimes (number of fitness evaluations) and their standard deviations over $100$ runs of the \oplea and the \oea on \onemax with noise rates $q = 0.01$ with varying problem size $n$.}
\label{plot:oplea_vs_oea_w_reevaluations}
\end{figure}




\section{Conclusion}
\label{sec:conclusion}

In this work, we have proven that both the \oplea and \oclea with at least logarithmic population sizes are very robust to noise, that is, with up to constant noise probabilities they optimize the \onemax benchmark in asymptotically the same time as if no noise was present. 
This significantly improves the state of the art for the \oplea, considerably reducing both the required population size and the runtime guarantee, and this is the first such result for the \oclea.

The reason for this progress is the general observation that the noisefree offspring can be seen as a biased uniform crossover between the parent and the noisy offspring. From this we proved that the true progress is at least a constant fraction of the noisy progress. 
The latter can be analyzed with known methods because the noisy offspring, asymptotically, has the distribution of an offspring obtained from bit-wise mutation with a rate that is the sum of the mutation and the noise rate. We are optimistic that this or analogous arguments will find applications in future runtime analyses again. 

Next steps to continue this line of research could include the first analysis of the \mclea~\cite{AntipovD21algo} in the presence of noise or an analysis of the algorithms studied in this work on the \leadingones benchmark, which generally is more affected by noise.

\subsubsection*{Acknowledgements}
This research benefited from the support of the FMJH Program Gaspard Monge for optimization and operations research and their interactions with data science. It was also supported by Australian Research Council through grant DP190103894.


\newcommand{\etalchar}[1]{$^{#1}$}

}

\end{document}